\documentclass[letterpaper,10pt]{ieeeconf}
\IEEEoverridecommandlockouts
\pdfoutput=1

\usepackage{amsmath,amssymb,mathrsfs,theorem}
\usepackage{graphicx,subfigure}
\usepackage{epsfig,color,psfrag}
\usepackage{xspace}

\usepackage[vlined,ruled,linesnumbered]{algorithm2e}
\SetKwInput{KwAssumes}{Assumes}
\SetKwInput{KwInput}{Input}
\SetKwInput{KwOutput}{Output}
\graphicspath{{./fig/}}

\newtheorem{theorem}{Theorem}[section]

\newtheorem{corollary}[theorem]{Corollary}
\newtheorem{definition}[theorem]{Definition}
\newtheorem{lemma}[theorem]{Lemma}

{\theorembodyfont{\rmfamily} 
\newtheorem{remark}[theorem]{Remark}

}

\def\QED{\mbox{\rule[0pt]{1.5ex}{1.5ex}}}

\newcommand{\LL}{\mathcal{L}}
\newcommand{\QQ}{\mathcal{Q}}

\newcommand{\PP}{\mathcal{P}}
\newcommand{\BB}{\mathcal{B}}

\newcommand{\RR}{\mathcal{R}}
\newcommand{\FIN}{\operatorname{FIN}}
\renewcommand{\FIN}{\mathbb{F}}

\newcommand{\real}{{\mathbb{R}}}
\newcommand{\subscr}[2]{{#1}_{\textup{#2}}}

\newcommand{\env}{\mathcal{E}}
\newcommand{\p}{\mathbf{p}}
\newcommand{\q}{\mathbf{q}}
\newcommand{\f}{\mathbf{f}}
\newcommand{\e}{\mathrm{e}}
\newcommand{\s}{\mathbf{s}}
\newcommand{\btsp}{\beta_{\mathrm{TSP}}}

\newcommand{\etsp}{\textup{ETSP}}
\newcommand{\emhp}{\textup{EMHP}}
\newcommand{\tmhp}{\textup{TMHP}}

\newcommand{\map}[3]{#1: #2 \rightarrow #3}

\newcommand{\expectation}[1]{\mbox{$\mathbb{E}\left[#1\right]$}} 
 
\newcommand{\prob}{\mathbb{P}}
\newcommand{\convertmap}{\operatorname{g}}

\newcommand{\capt}{\textup{capt}}
\newcommand{\esc}{\textup{esc}}
\newcommand{\erf}{\operatorname{erf}}
\newcommand{\capfrac}{\subscr{\mathbb{F}}{cap}}

\newcommand\oprocendsymbol{\hbox{$\square$}}
\newcommand\oprocend{\relax\ifmmode\else\unskip\hfill\fi\oprocendsymbol}

\title{A Dynamic Boundary Guarding Problem with Translating
  Targets\thanks{This material is based upon work supported in part by
    ARO-MURI Award W911NF-05-1-0219 and ONR Award N00014-07-1-0721.}}

\author{Stephen L. Smith\thanks{S. L. Smith, S. D. Bopardikar, and
    F. Bullo are all with the Center for Control, Dynamical Systems and
    Computation, University of California at Santa Barbara, Santa Barbara,
    CA 93106, USA. Email:
    \texttt{\{stephen,shaunak,bullo\}@engineering.ucsb.edu.}} \quad
  Shaunak D. Bopardikar \quad
  Francesco Bullo}

\begin{document}
\maketitle
\begin{abstract}
  We introduce a problem in which a service vehicle seeks to guard a
  deadline (boundary) from dynamically arriving mobile targets. The
  environment is a rectangle and the deadline is one of its edges.
  Targets arrive continuously over time on the edge opposite the
  deadline, and move towards the deadline at a fixed speed. The goal
  for the vehicle is to maximize the fraction of targets that are
  captured before reaching the deadline. We consider two cases; when
  the service vehicle is faster than the targets, and; when the
  service vehicle is slower than the targets. In the first case we
  develop a novel vehicle policy based on computing longest paths in a
  directed acyclic graph. We give a lower bound on the capture
  fraction of the policy and show that the policy is optimal when the
  distance between the target arrival edge and deadline becomes very
  large. We present numerical results which suggest near optimal
  performance away from this limiting regime. In the second case, when
  the targets are slower than the vehicle, we propose a policy based
  on servicing fractions of the translational minimum Hamiltonian
  path. In the limit of low target speed and high arrival rate, the
  capture fraction of this policy is within a small constant factor of
  the optimal. \end{abstract}

\section{Introduction}

Vehicle motion planning in dynamic environments arises in many
important autonomous vehicle applications.  In areas such as
environmental monitoring, surveillance and perimeter defence, the
vehicle must re-plan its motion as it acquires information on its
surroundings.  In addition, remote operators may add tasks to, or
remove tasks from, the vehicle's mission in real-time.  In this paper
we consider a problem in which a vehicle must defend a boundary in a
dynamic environment with approaching targets.

Static vehicle routing problems consider planning a path through a
fixed number of locations. Examples include the traveling salesperson
problem (TSP)~\cite{NC:75}, the deadline-TSP and vehicle routing with
time-windows~\cite{NB-AB-DK-AM:04}. Recently, researchers have looked
at the TSP with moving objects. In~\cite{PC-RM:99} the authors
consider objects moving on straight lines and focus on the case when
the objects are slower than the vehicle and when the vehicle moves
parallel to the $x$- or $y$-axis. The same problem is studied
in~\cite{MH-BJN:02}, but with arbitrary vehicle motion, and it is
called the translational TSP. The authors of~\cite{MH-BJN:02} propose
a polynomial-time approximation scheme to catch all objects in minimum
time. Other variations of the problem are studied
in~\cite{CSH-GR-AZ:03} and~\cite{YA-EM-SS:08}.

Dynamic vehicle routing (DVR) is a class of problems in which vehicles
must plan paths through service demand locations that arrive
sequentially over time. An early DVR problem was the dynamic traveling
repairperson problem~\cite{HNP:88,DJS-GJvR:91}, where each demand
assumes a fixed location upon arrival, and the vehicle must spend some
amount of on-site service time at each location. This problem has also
been studied from the online algorithm
perspective~\cite{AB-REY:98,SOK-WEP-DP-LS:03}. Other recent DVR
problems include DVR with demands that disappear if left unserviced
for a certain amount of time~\cite{MP-NB-EF-VI:08}, and demands with
different priority levels~\cite{SLS-MP-FB-EF:09a}. In our earlier
work~\cite{SDB-SLS-FB:08v}, we introduced a DVR problem in which
demands arrive on a line segment and move in a perpendicular direction
at a fixed speed slower than the vehicle. We derived conditions on the
demand arrival rate and demand speed for the existence of a vehicle
routing policy which can serve all demands, and a proposed a policy
based on the translational minimum Hamiltonian path.

\paragraph*{Contributions}
In this paper we introduce the following problem (see
Fig.~\ref{fig:problem_setup}): Targets (or demands) arrive according
to a stochastic process on a line segment of length $W$.  Upon arrival
the demands move with fixed speed $v$ towards a deadline which is at a
distance $L$ from the generator. A unit speed service vehicle seeks to
capture the demands before they reach the deadline (i.e., within $L/v$
time units of being generated).  The performance metric is the
fraction of demands that are captured before reaching the deadline.
\begin{figure}
\includegraphics[width=.55\linewidth]{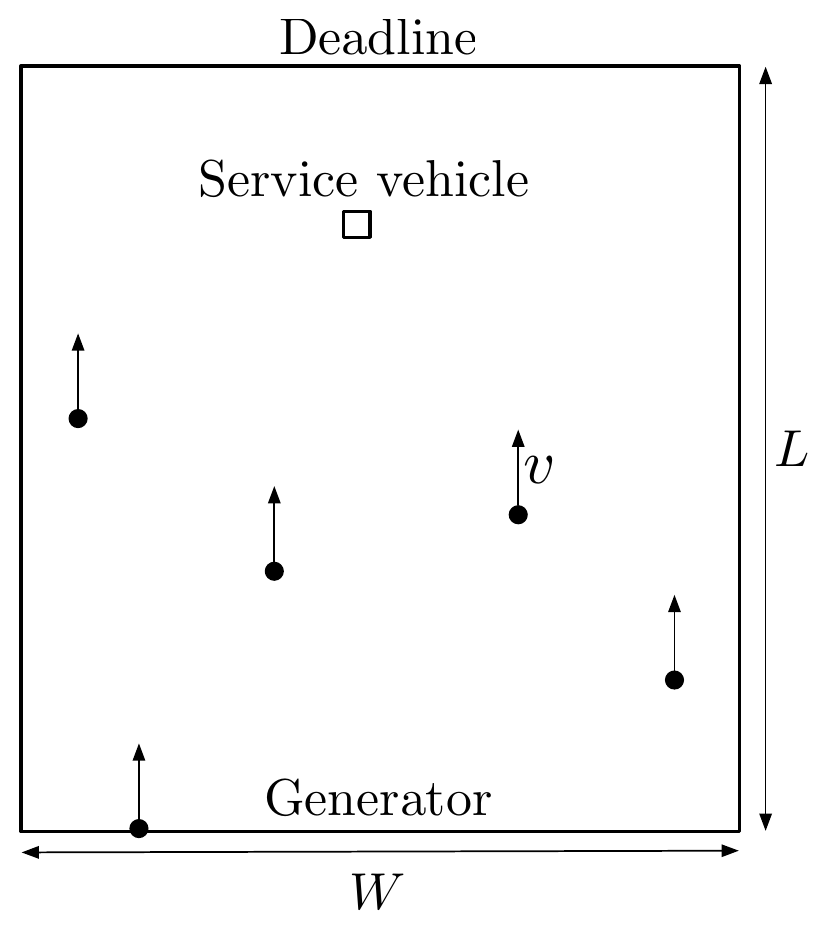}
\centering 
\caption{The problem setup.  Demands are shown as black disks
  approaching the deadline at speed $v$. The service vehicle is a
  square.}
\label{fig:problem_setup}
\end{figure}

We assume that the arrival process is uniform along the line segment
and temporally Poisson with rate $\lambda$.  In the case when the
demands are faster than the service vehicle (i.e., $v \geq 1$) we
introduce the novel Longest Path policy, which is based on computing
longest paths in a directed acyclic \emph{reachability graph}. When $L
\geq vW$, we derive a lower bound on the capture fraction as a
function of the system parameters. We show that the Longest Path
policy is the optimal policy when $L$ is much greater than $vW$. In
the case when the demands are slower than the service vehicle (i.e,
$v<1$), we propose a policy based on the translational minimum
Hamiltonian path called the TMHP-fraction policy. In the limit of low
demand speed and high arrival rate, the capture fraction of this
policy is within a small constant factor of the optimal. We present
numerical simulations which verify our results, and show that the
Longest Path policy performs very near the optimal even when $L < vW$.

The paper is organized as follows. In Section~\ref{sec:prob_form} we
formulate the problem and in Section~\ref{sec:prelims} we review some
background material. In Section~\ref{sec:fast_demand} we consider the
case of $v \geq 1$ and introduce the Longest Path Policy. In
Section~\ref{sec:slow_demand} we study $v <1$ and introduce the
TMHP-fraction policy. Finally, in Section~\ref{sec:simu} we present
simulations results.

\section{Problem Formulation}
\label{sec:prob_form}

Consider an environment $\env:=[0,W]\times[0,L] \subset \real^2$ as
shown in Figure~\ref{fig:problem_setup}.  The line segment
$[0,W]\times \{0\}\subset \env$ is termed the \emph{generator}, and
the segment $[0,W]\times \{L\}\subset \env$ is termed the
\emph{deadline}.  The environment contains a single vehicle with
position $\p(t) = [X(t), Y(t)]^T \in \env$, modeled as a first-order
integrator with unit speed.  Demands (or targets) arrive in the
environment according to a temporal Poisson process with rate $\lambda
> 0$.  Upon arrival, each demand assumes a uniformly distributed
location on the generator, and then moves with constant speed $v >0$
in the positive $y$-direction towards the deadline.  If the vehicle
intercepts a demand before the demand reaches the deadline, then the
demand is captured.  On the other hand, if the demand reaches the
deadline before being intercepted by the vehicle, then the demand
escapes.  Thus, to capture a demand, it must be intercepted within
$L/v$ time units of being generated.

We let $\QQ(t) \subset \env$ denote the set of all outstanding demand
locations at time $t$.  If the $i$th demand to arrive is captured,
then it is removed from $\QQ$ and placed in the set $\QQ_{\capt}$ with
cardinality $n_{\capt}$.  If the $i$th demand escapes, then it is
removed from $\QQ$ and placed in $\QQ_{\esc}$ with cardinality
$n_{\esc}$.
\

\paragraph*{Causal Policy}
A causal feedback control policy for the vehicle is a map
$\map{\PP}{\env\times \FIN(\env)}{\real^2}$, where $\FIN(\env)$ is the
set of finite subsets of $\env$, assigning a commanded velocity to the
service vehicle as a function of the current state of the system:
$\dot{\p}(t)=\PP(\p(t),\QQ(t))$.

\paragraph*{Non-causal Policy}
In a non-causal feedback control policy the commanded velocity of the
service vehicle is a function of the current and future state of the
system.  Such policies are not physically realizable, but they will
prove useful in the upcoming analysis.

Formally, let the generation of demands commence at time $t=0$, and
consider the sequence of demands $(q_1,q_2,\ldots)$ arriving at
increasing times $(t_1,t_2,\ldots)$, with $x$-coordinates
$(x_1,x_2,\ldots)$.  We can also model the arrival process by assuming
that at time $t=0$, all demands are located in
$[0,W]\times(-\infty,0]$, move in the $y$-direction at speed $v$ for
all $t > 0$, and are revealed to the service vehicle when they cross
the generator.  Thus, at time $t=0$, the position of the $i$th demand
is $(x_i,v(t-t_i))$.  We can define a set containing the position of
all demands in the region $[0,W]\times(-\infty,0]$ at time $t$ as
$\QQ_{\mathrm{unarrived}}(t)$.  Then, a non-causal policy is one for
which $\dot{\p}(t)=\PP(\p(t),\QQ(t)\cup\QQ_{\mathrm{unarrived}}(t))$.

\paragraph*{Problem Statement}
The goal in this paper is to find causal policies $P$ that maximize
the fraction of demands that are captured $\capfrac(P)$, termed the
\emph{capture fraction}, where
\[
\capfrac(P) :=
\limsup_{t\to+\infty}\expectation{\frac{n_{\capt}(t)}{n_{\capt}(t)+n_{\esc}(t)}}.
\]

\section{Preliminary Combinatorial Results}
\label{sec:prelims}

We now review the longest path problem, the distribution of demands in
an unserviced region, and optimal tours/paths through a set of points.

\subsection{Longest Paths in Directed Acyclic Graphs}
\label{sec:DAG_longest}

A directed graph $G=(V,E)$ consists of a set of vertices $V$ and a set
of directed edges $E\subset V\times V$.  An edge $(v,w)\in E$ is
directed from vertex $v$ to vertex $w$.  A \emph{path} in $G$ is a
sequence of vertices such that from each vertex in the sequence, there
is an edge in $E$ directed to the next vertex in the sequence.  A path
is \emph{simple} if it contains no repeated vertices. A \emph{cycle}
is a path in which the first and last vertex in the sequence are the
same.  A graph $G$ is \emph{acyclic} if it contains no cycles.  The
longest path problem is to find a simple path of maximum length (i.e.,
a path that visits a maximum number of vertices).  In general this
problem is NP-hard as its solution would imply a solution to the
well known Hamiltonian path problem~\cite{BK-JV:07}.  However, if the
graph is a DAG, then the longest path problem has an efficient dynamic
programming solution~\cite{NC:75} with complexity $O(|V| +|E|)$, that
relies on topologically sorting~\cite{THC-CEL-RLR-CS:01} the vertices.

\subsection{Distribution of Demands in an Unserviced Region}
\label{sec:poisson_dist}

Demands arrive uniformly on the generator, according to a Poisson
process with rate $\lambda$.  The following lemma describes the
distribution of demands in an unserviced region.  For a finite set
$\QQ$, we let $|\QQ|$ denote its cardinality.

\begin{lemma}[Distribution of outstanding demands,
  \cite{SDB-SLS-FB:08v}]
  \label{lem:poisson}
  
  Suppose the generation of demands commences at time $0$ and no demands
  are serviced in the interval $[0, t]$.  Let $\QQ$ denote the set of all
  demands in $[0,W]\times [0,vt]$ at time $t$.  Then, given a measurable
  compact region $\RR$ of area $A$ contained in $[0,W]\times [0,vt]$,
  \[
  \prob[|\RR \cap \QQ| = n] = \frac{\mathrm{e}^{-\bar\lambda
      A}(\bar\lambda A)^n}{n!}, \quad \text{where }\bar \lambda :=
  \lambda/(v W).
  \]
\end{lemma}

The previous lemma tells us that number of demands $N$ in an serviced
region of area $A$ is a Poisson distributed with parameter $\lambda
A/(vW)$.  In addition, conditioned on $N$, the demands are
independently and uniformly distributed.

\subsection{The Euclidean Shortest Path/Tour Problems}
\label{sec:ETSP_EMHP}

Given a set $\QQ$ of $n$ points in $\real^2$, the Euclidean traveling
salesperson problem (ETSP) is to find the minimum-length tour (i.e.,
cycle) of $\QQ$. Letting $\etsp(\QQ)$ denote the minimum length of a
tour of $\QQ$, we can state the following result.

\begin{theorem}[Length of ETSP tour, \cite{JB-JH-JH:59}]
  \label{thm:etsp_length}
  Consider a set $\QQ$ of $n$ points independently and uniformly
  distributed in a compact set $\env$ of area $|\env|$.  Then, there
  exists a constant $\btsp$ such that, with probability one,
\begin{equation}
\label{eq:tspd}
\lim_{n\rightarrow+\infty} \frac{\etsp(\QQ)}{\sqrt{n}} =
\beta_{\mathrm{TSP}} \sqrt{|\env|}.
\end{equation}
\end{theorem}

The constant $\btsp$ has been estimated numerically as $\btsp \approx
0.7120\pm 0.0002$,~\cite{AGP-OCM:96}.  

The Euclidean Minimum Hamiltonian Path (EMHP) problem is to compute
the shortest path through a set of points. In this paper we consider a
constrained EMHP problem:
  Given a start point $\s$, a set of $n$ points $\QQ$, and a finish
  point $\f$, all in $\real^2$, determine the shortest path which
  starts at $\s$, visits each point in $\QQ$ exactly once, and
  terminates at $\f$.
  We let $\emhp(\s,\QQ,\f)$ denote the length of the shortest path.
  \begin{corollary}[Length of EMHP]
    \label{cor:emhp_length}
    Consider a set $\QQ$ of $n$ points independently and uniformly
    distributed in a compact set $\env$ of area $|\env|$, and any two
    points $\s,\f\in \env$.  Then with probability one,
\[
\lim_{n\to +\infty}\frac{\emhp(\s,\QQ,\f)}{\sqrt{n}} = \btsp
\sqrt{|\env|},
\]
where $\btsp$ is defined in Theorem~\ref{thm:etsp_length}.
  \end{corollary}

  The above corollary states that the length of the EMHP and the ETSP
  tour are asymptotically equal, and it follows directly from the fact
  that as $n\to+\infty$, the diameter of $\env$ is negligible when
  compared to the length of the tour/path.

\subsection{Translational Minimum Hamiltonian Path (TMHP)}
\label{sec:TMHP_rev}

The TMHP problem is posed as follows.
  Given initial coordinates; $\s$ of a start point,
  $\QQ:=\{\q_1,\dots,\q_n\}$ of a set of points, and $\f$ of a finish
  point, all moving with speed $v\in]0,1[$ in the positive
  $y$-direction, determine a minimum length path that starts at time
  zero from point $\s$, visits all points in the set $\QQ$ and ends at
  the finish point. The following gives a solution~\cite{MH-BJN:02} for the TMHP problem. \\%
(i) Define the map $\map{\convertmap}{\real^2}{\real^2}$ by
\begin{equation*}
\convertmap(x,y) = \Big(\frac{x}{\sqrt{1-v^2}},\frac{y}{1-v^2}\Big).
\end{equation*}
(ii) Compute the EMHP that starts at $\convertmap_v(\s)$, passes
through
$\{\convertmap(\q_1),\dots,\convertmap(\q_n)\}=:\convertmap(\QQ)$ and
ends at $\convertmap(\f)$. \\*[0.1cm]%
(iii) To reach a translating point with initial position $(x,y)$ from
the initial position $(X,Y)$, move towards the point $(x,y+vT)$, where
  \begin{equation*}
    \label{eq:travel_time}
    T = \frac{\sqrt{(1-v^2)(X-x)^2 + (Y-y)^2}}{1-v^2} - 
    \frac{v(Y-y)}{1-v^2}.
   \end{equation*}
The length $\tmhp_v(\s,\QQ,\f)$ of the path is as follows.
\begin{lemma}[TMHP length, \cite{MH-BJN:02}]
  \label{lem:ttsp}
  Let the initial coordinates $\s=(x_{\s},y_{\s})$ and
  $\f=(x_{\f},y_{\f})$, and the speed of the points $v\in{]0,1[}$. Then,
  \[
  \tmhp_v(\s,\QQ,\f) =
  \emhp(\convertmap(\s),\convertmap(\QQ),\convertmap(\f)) +
  \frac{v(y_{\f}-y_{\s})}{1-v^2}.
  \]
\end{lemma}

\section{Demand Speed Greater Than Vehicle Speed}
\label{sec:fast_demand}

Here we develop a policy for the case when the demand speed $v \geq
1$.  In this policy, the service vehicle remains on the deadline and
services demands as per the longest path in a directed acyclic
reachability graph.  In this section we begin by introducing the
reachability graph, and then proceed to state and analyze the Longest
Path policy.

\subsection{Reachable Demands}

Consider a demand generated at time $t_1 \geq 0$ at position $(x,0)$.
The demand moves in the positive $y$-direction at speed $v \geq 1$,
and thus $(x(t),y(t)) = (x,v(t-t_1))$ for each $t\in[t_1,T]$, where
$T$ is either the time of escape (i.e., $T = L/v +t_1$), or it is the
time of capture.  Now, given the service vehicle location
$(X(t),Y(t))$, a demand with position $(x,y(t))$ is reachable if and
only if
\begin{equation}
\label{eq:capture_cond}
v|X(t)-x| \leq Y(t)- y(t).
\end{equation}
That is, the service vehicle must be at a height of at least
$v|X(t)-x|$ above the demand in order to capture it.
\begin{definition}[Reachable set]
  The reachable set from a position $(X,Y)\in \env$ is
  \[
  R(X,Y) := \{(x,y)\in \env : v|X-x| \leq |Y - y|\}.
  \]
  If the service vehicle is located at $(X,Y)$, then a demand can be
  captured if and only if it lies in the set $R(X,Y)$.
\end{definition}
An example of the reachable set is shown in
Figure~\ref{fig:reachability}.
\begin{figure}
\centering 
\includegraphics[width=.47\linewidth]{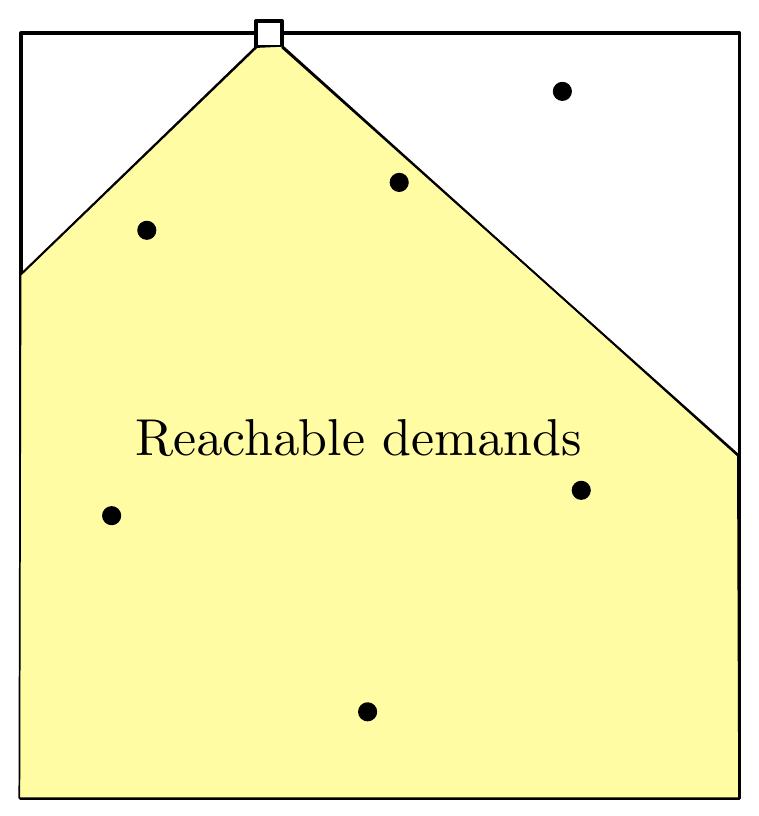}
\hfill
\includegraphics[width=.47\linewidth]{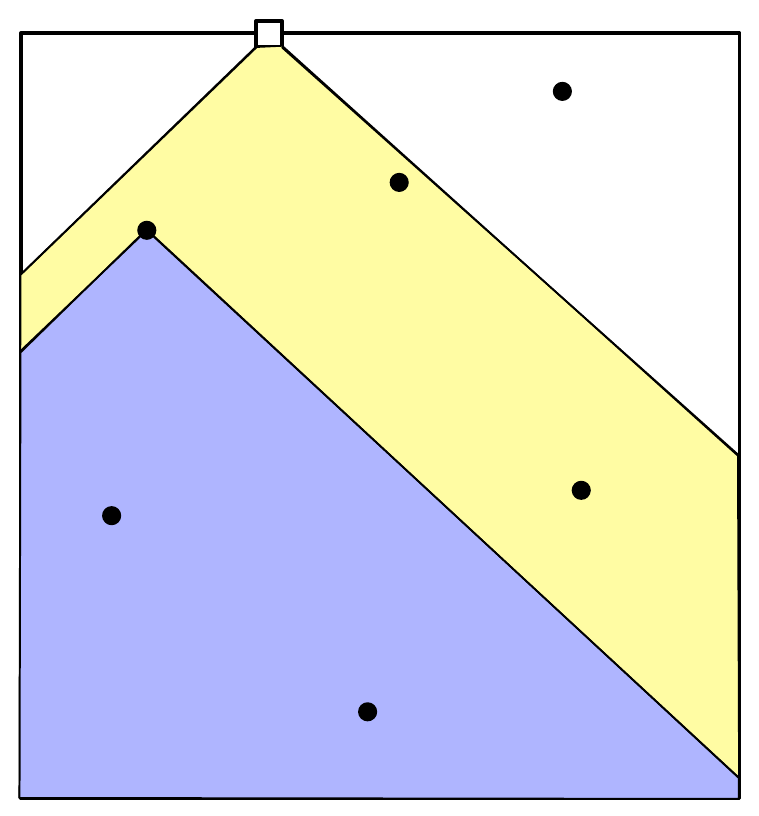} \\
\includegraphics[width=.47\linewidth]{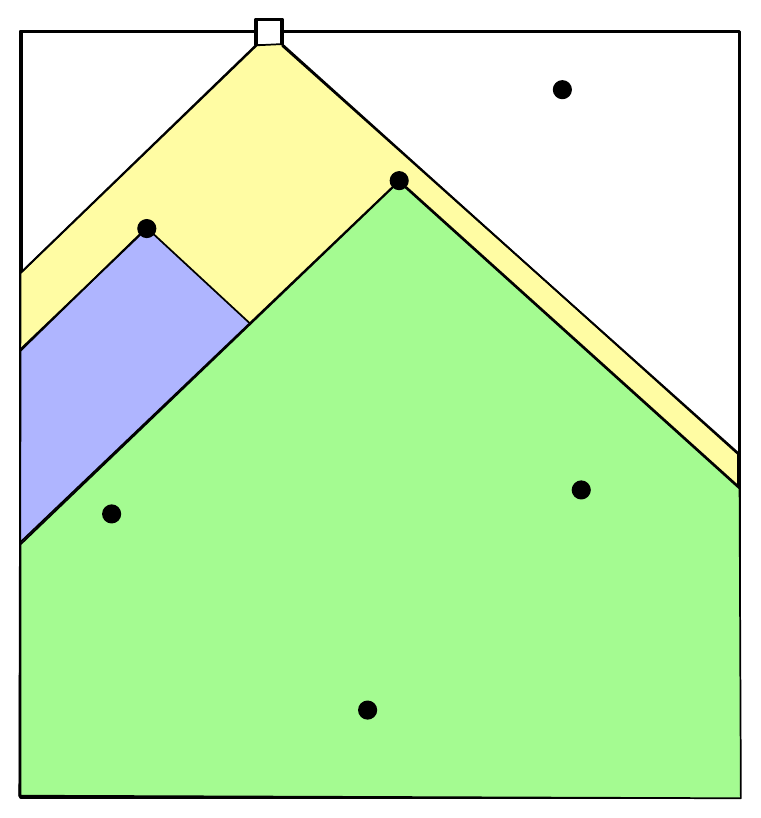}
\hfill
\includegraphics[width=.47\linewidth]{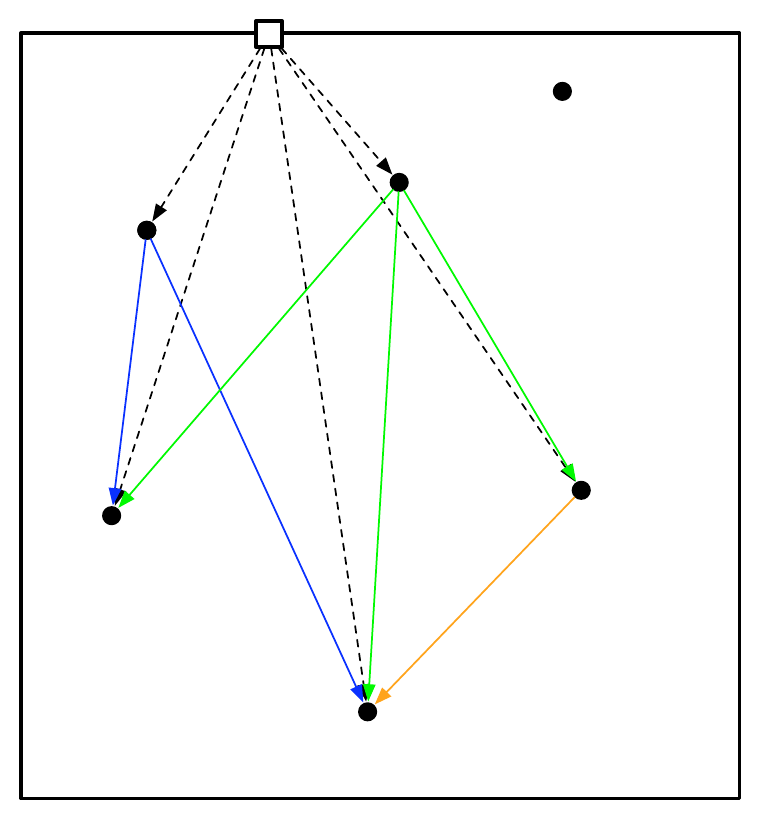}
\caption{The construction of the reachability graph.  The top-left
  figure shows the set of reachable points from a vehicle positioned
  on the deadline.  The top-right and bottom-left figures show the
  reachable set from demand locations.  The bottom-right figure shows
  the reachability graph.} %
\label{fig:reachability} %
\end{figure}%
Next, given a demand in the reachable set, the following motion gives
a method of capture.
\begin{definition}[Intercept motion]
  Consider a vehicle position $((X(\bar t),Y(\bar t))$ and a demand
  position $(x,y(\bar t))\in R(X(\bar t),Y(\bar t))$ at time $\bar t
  \geq 0$.  In intercept motion, the service vehicle captures the
  demand by first moving horizontally at unit speed to the position
  $(x_i,Y(\bar t))$, and then waiting at the location for the demands
  arrival.
\end{definition}

\begin{lemma}[Optimality of intercept motion]
  \label{lem:opt:intercept}
  Consider $v\geq 1$, and let the service vehicle be initially
  positioned on the deadline.  Then, there is an optimal policy in
  which the service vehicle uses only intercept motion.
\end{lemma}
\begin{proof}
  Let the service vehicle be positioned at $(X,L)$, and consider a
  demand at $(x,y)\in R(X,L)$.  From equation~(\ref{eq:capture_cond}),
  we have $v|X-x| \leq L - y$.  If $v|X-x| = L - y$, then deadline
  motion is the only way in which the demand can be captured. Thus,
  assume that $v|X-x| < L - y$, and consider two cases; Case~1 in
  which intercept motion is used, and Case~2 in which the demand is
  captured at a location $(x,Y)$, where $Y < L$.

  Notice that the position of each outstanding demand relative to the
  service vehicle position at capture is the same in Case~1 as in
  Case~2. Thus, the reachable set in Case~2 is a strict subset of
  reachable set in Case~1 and the vehicle gains no advantage by moving
  off of the deadline.
\end{proof}

Next, consider the set of demands in $R(X(\bar t),Y(\bar t))$, and
suppose the vehicle chooses to capture demand $i$, with position
$\q_i(\bar t) = (x_i,y_i(\bar t))\in R(X(\bar t),Y(\bar t))$.  Upon
capture at time $T$, the service vehicle can recompute the reachable
set, and select a demand that lies within.  Since all demands
translate together, every demand that was reachable from $\q_i(\bar
t)$, is reachable from $\q_i(T)$.  Thus, the service vehicle can
``look ahead'' and compute the demands that will be reachable from
each captured demand position.  This idea motivates the concept of a
reachability graph.

\begin{definition}[Reachability graph]
  For $v\geq 1$, the reachability graph of a set of points
  $\{\q_1,\ldots,\q_n\}\in\env$, is a directed acyclic graph with
  vertex set $V :=\{1,\ldots,n\}$, and edge set $E$, where for $i,j\in
  V$, the edge $(i,j)$ is in $E$ if and only if $\q_j\in R(\q_i)$ and
  $j \neq i$.
\end{definition}

Given a set $\QQ$ of $n$ outstanding demands, and a vehicle position
$(X,Y)$, we can compute the corresponding reachability graph (see
Fig.~\ref{fig:reachability}) in $O(n^2)$ computation time.  In
addition, by Section~\ref{sec:DAG_longest} we can compute the longest
path in a reachability graph in $O(n^2)$ computation time.

\subsection{A Non-causal Policy and Upper Bound}

To derive an upper bound for $v \geq 1$, we begin by considering a
non-causal policy.  In the online algorithms literature, such a policy
would be referred to as an \emph{offline algorithm}~\cite{AB-REY:98}.
Figure~\ref{fig:noncausal_path} shows an example of a path generated
by the Non-causal Longest Path policy.  Note that the service vehicle
will intercept each demand on the deadline, and thus the path depicts
which demands will be captured, and in what order.

\begin{algorithm}[H] 
  \dontprintsemicolon %
  \nocaptionofalgo %
  \KwAssumes{Vehicle is located on deadline and $v \geq 1$.} %
  Compute the reachability graph of the vehicle position and all
  demands in $\QQ(0)\cup\QQ_{\mathrm{unarrived}}(0)$. \; %
  Compute a longest path in this graph, starting at the service
  vehicle location. \; %
  Capture demands in the order they appear on the path, intercepting
  each demand on the deadline.  \; %
\caption{\bf Non-causal Longest Path (NCLP) policy}
\end{algorithm}

\begin{figure}
\includegraphics[width=.42\linewidth]{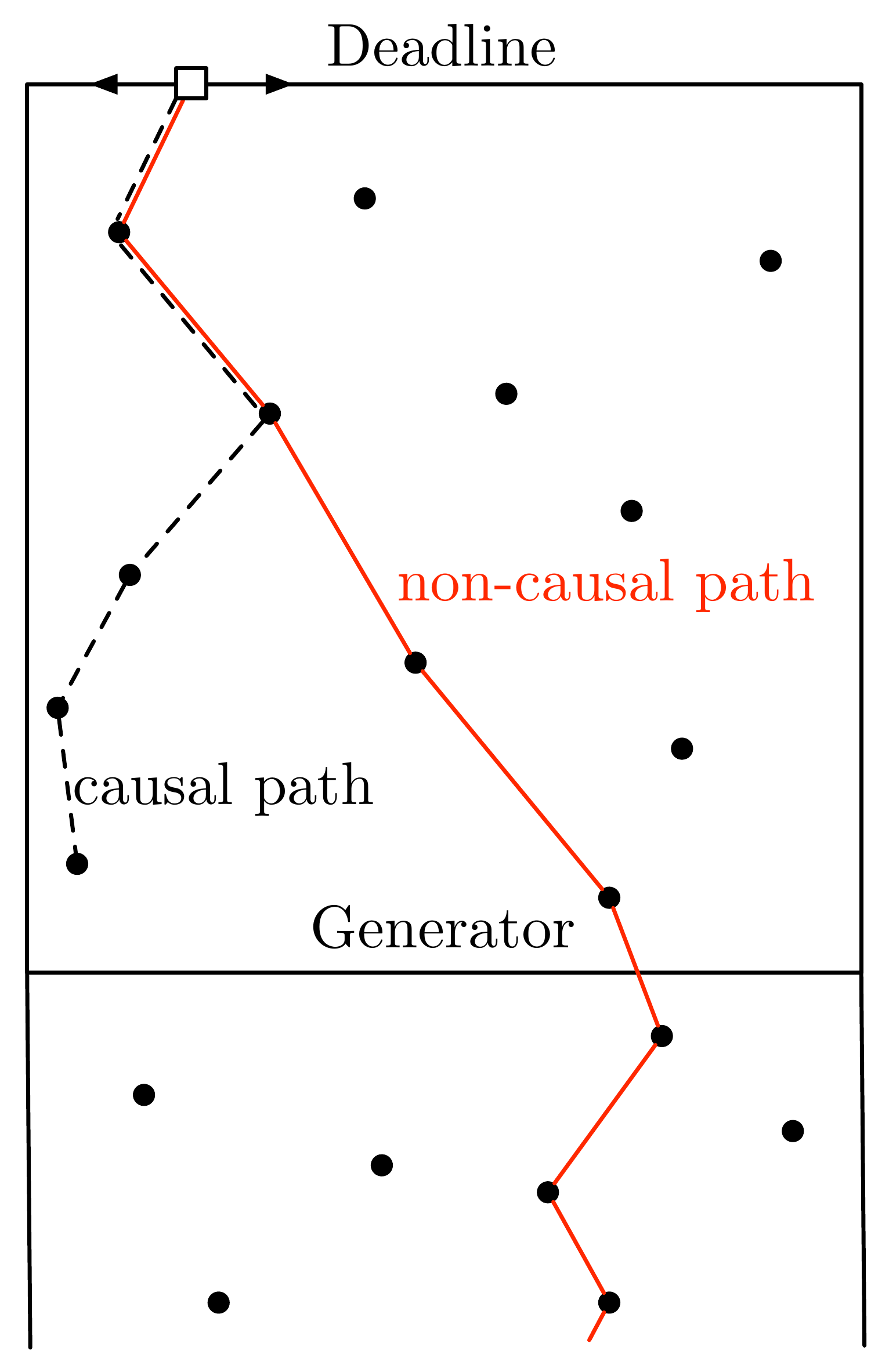}
\centering 
\caption{A snapshot in the evolution of the Non-causal Longest Path
  policy.  The vehicle has planned the solid red path through all
  demands, including those that have not yet arrived.  In comparison,
  a dashed causal longest path is shown, which only considers demands
  that have arrived.}
\label{fig:noncausal_path}
\end{figure}

\begin{lemma}[Optimal non-causal policy]
  If $v\geq 1$, then the Non-causal Longest Path policy is an optimal
  non-causal policy.  Moreover, if $v\geq 1$, then for every causal
  policy $P$,
  \[
  \capfrac(P) \leq \capfrac(\mathrm{NCLP}).
  \]
\end{lemma}
\begin{proof}
  The reachability graph $\QQ(0)\cup\QQ_{\mathrm{unarrived}}(0)$
  contains every possible path that the service vehicle can follow.
  When $v\geq 1$ the graph is a directed acyclic graph and thus the
  longest path (i.e., the path which visits the most vertices in the
  graph) is well defined.  The vehicle uses intercept motion, and thus
  by Lemma~\ref{lem:opt:intercept} the NCLP policy is an optimal
  non-causal policy, and its capture fraction upper bounds every
  causal policy.
\end{proof}

\subsection{The Longest Path Policy}

We now introduce the Longest Path policy.  In the LP policy, the
fraction $\eta$ is a design parameter. The lower $\eta$ is chosen, the
better the performance of the policy, but this comes at the expense of
increased computation.

\begin{algorithm}[H] 
  \dontprintsemicolon %
  \nocaptionofalgo %
  \KwAssumes{Vehicle is located on deadline and $v \geq 1$} %
  Compute the reachability graph of the vehicle position and all
  demands in $\QQ(0)$. \; %
  Compute a longest path in this graph, starting at the service
  vehicle location. \; %
  Capture demands in the order they appear on the path, intercepting
  each demand on the deadline.  \; %
  Once a fraction $\eta\in{]0,1]}$ of the demands on the path have
  been serviced, recompute the reachability graph of all outstanding
  demands and return to step 2. \; %
\caption{\bf The Longest Path (LP) policy}
\end{algorithm}

In the following theorem, we relate the Longest Path policy to its
non-causal relative.  Such a bound is referred to as a
\emph{competitive ratio} in the online algorithms
literature~\cite{AB-REY:98}.

\begin{theorem}[Optimality of Longest Path policy]
\label{thm:LP_to_NCLP}
If $v\geq 1$, then
\[
\capfrac(\mathrm{LP}) \geq \left(1 -
  \frac{vW}{L}\right)\capfrac(\mathrm{NCLP}),
\]
and thus the LP policy is optimal as $vW/L \to +\infty$.
\end{theorem}

\begin{proof}
  Suppose that the generation of demands begins at $t=0$ and let us
  consider two scenarios; (a) the vehicle uses the Longest Path
  policy, and (b) the vehicle uses the Non-causal Longest Path policy.
  Then, at any instant in time $t_1 >0$ we can compare the number of
  demands captured in scenario (a) to the number captured in scenario
  (b) (refer to Fig.~\ref{fig:scenario_ab}).
  
  \begin{figure}
    \includegraphics[width=.45\linewidth]{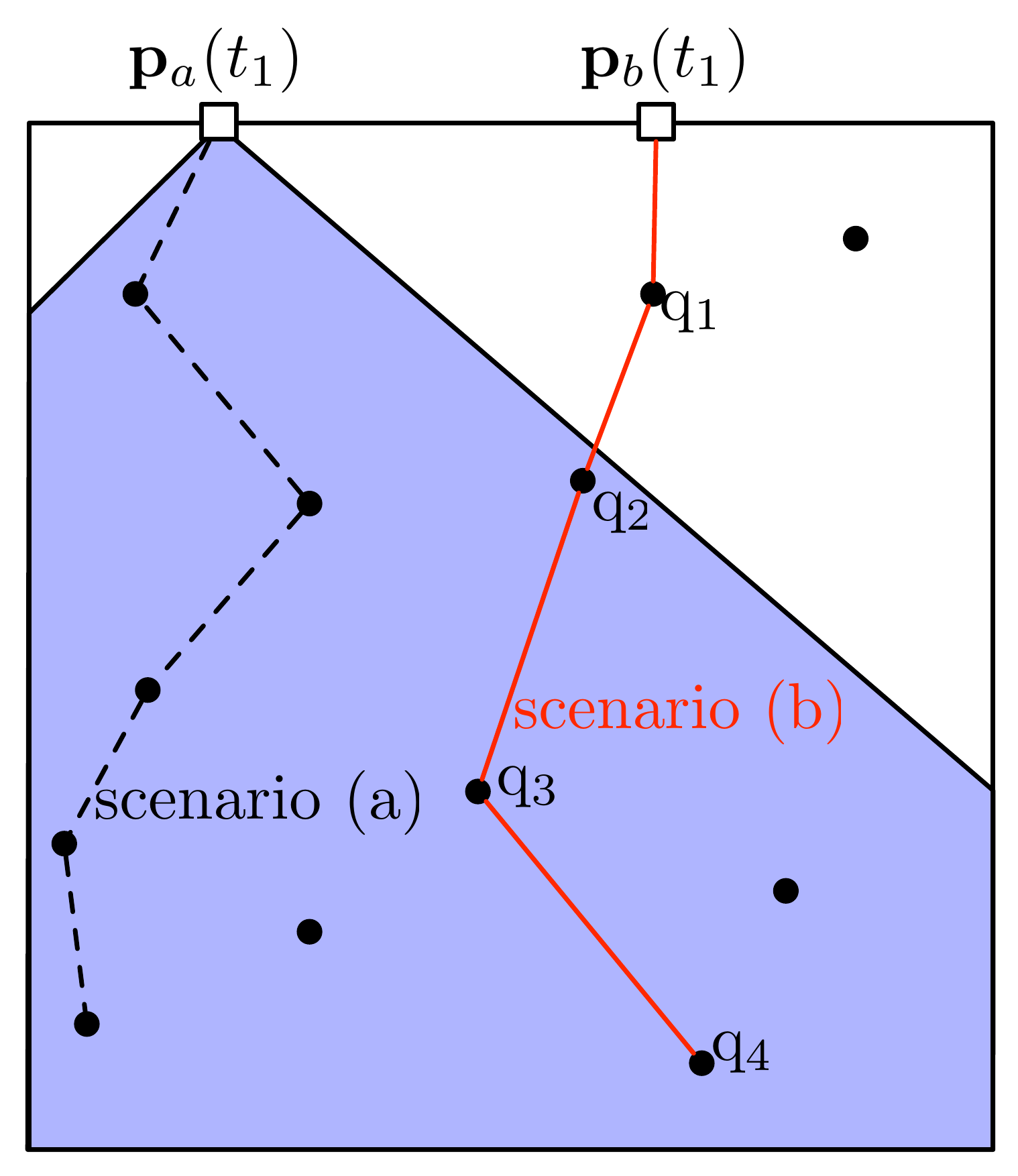}
    \centering
    \caption{Scenario (a) and (b) for the proof of
      Theorem~\ref{thm:LP_to_NCLP}.  Path (a) visits five demands and
      thus $\LL_a=5$.  Path (b) visits four demands, yielding $m=4$.
      The demand $\q_2$ is the highest on path (b) that can be
      captured from $\p_a(t_1)$.  Thus, $n=1$, and $5=\LL_a > m-n =
      3$.}
    \label{fig:scenario_ab}
  \end{figure}

  Let us consider a time instant $t_1$ where in scenario (a), the
  vehicle is recomputing the longest path through all outstanding
  demands $\QQ(t_1)$.  Let us denote by $\p_a(t_1)$ and $\p_b(t_1)$,
  the vehicle position in scenario (a) and scenario (b), respectively,
  at time $t_1$.  In scenario (b), let the path that the vehicle will
  take through $\QQ(t_1)$ be given by $(\q_1,\q_2,\ldots,\q_m)\in
  \QQ(t_1)$.  The demand $\q_1$ is reachable from $\p_b(t_1)$, but it
  may not be reachable from $\p_a(t_1)$.  However, a lower bound on
  the length of the longest path in scenario (a) is:
  $(\q_{n+1},\q_{n+2},\ldots,\q_m)$, where $\q_{n+1}$, $n\in
  \{0,\ldots,m-1\}$, is the highest demand that can be captured from
  $\p_a(t_1)$, see Fig.~\ref{fig:scenario_ab}.  Thus, the length of
  the longest path in scenario (a), $\LL_a$, is at least
  \begin{equation}
    \label{eq:length_sc_a}
  \LL_a \geq m - n,
  \end{equation}
  where $m$ is the length of the path in scenario (b).

  Now, since the deadline has width $W$, the vehicle in scenario (a)
  can capture any demand $(x,y)$ with $y \leq L-vW$. Thus, the demands
  $\q_1,\ldots,\q_n$ must all have $y$-coordinates in ${]L-vW,L]}$.  Let
  the total number of outstanding demands at time $t_1$ be
  $\subscr{N}{tot}$.  Then, conditioned on $\subscr{N}{tot}$, by
  Lemma~\ref{lem:poisson}, the expected number of outstanding demands
  contained in $[0,W]\times {]L-vW,L]}$ is $\subscr{N}{tot} vW/L$.
  Hence,
  \begin{equation}
    \label{eq:n_expect}
  \expectation{n|\subscr{N}{tot}} = \subscr{N}{tot} \frac{vW}{L}
  \capfrac(\mathrm{NCLP}).
  \end{equation}
  Similarly, for the length of the path through $\QQ(t_1)$ in scenario
  (b), we have
  \begin{equation}
    \label{eq:m_expect}
    \expectation{m|\subscr{N}{tot}} = \subscr{N}{tot}
    \capfrac(\mathrm{NCLP}).
  \end{equation}
  Combining equations~(\ref{eq:n_expect}) and~(\ref{eq:m_expect}) with
  equation~(\ref{eq:length_sc_a}) we obtain
  \begin{align*}
    \expectation{\LL_a|\subscr{N}{tot}} &\geq \subscr{N}{tot}
    \left(1-\frac{v W}{L}\right) \capfrac(\mathrm{NCLP}), \\
    \expectation{\frac{\LL_a}{\subscr{N}{tot}}|\subscr{N}{tot}} &\geq
    \left(1-\frac{v W}{L}\right) \capfrac(\mathrm{NCLP}).
  \end{align*}
  But $\LL_a/\subscr{N}{tot}$ is the fraction of outstanding demands
  in $\QQ(t_1)$ that will be captured in scenario (a), and it does not
  depend on the value of $\subscr{N}{tot}$.  By the law of total
  expectation
  \[
  \expectation{\frac{\LL_a}{\subscr{N}{tot}}} =
  \expectation{\expectation{\frac{\LL_a}{\subscr{N}{tot}}|\subscr{N}{tot}}}
  \geq \left(1-\frac{v W}{L}\right)
  \capfrac(\mathrm{NCLP}).
  \]
  Each time the
  longest path is recomputed, the path in scenario (a) will capture at
  least this fraction of demands. Thus, we have
  $\capfrac(\mathrm{LP}) \geq
  \expectation{\LL_a/\subscr{N}{tot}}$ and have proved the result.
\end{proof}
\begin{remark}[Conservativeness of bound]
  The bound in Theorem~\ref{thm:LP_to_NCLP} is conservative. This is
  primarily due to bounding the expected distance between the causal
  and non-causal paths by $W$. The distance between two independently
  and uniformly distributed points in $[0,W]$, is $W/3$. The distance
  is even less if the points are positively correlated (as is likely
  the case for the distance between paths). Thus, it seems that it may
  be possible to increase the bound to
  \[
  \capfrac(\mathrm{LP}) \geq \left(1 -
  \frac{vd}{L}\right)\capfrac(\mathrm{NCLP}),
  \]
  where $d < W/3$. \oprocend
\end{remark}

The previous theorem establishes the performance of the Longest Path
policy relative to a non-causal policy.  However, the LP policy is
difficult analyze directly.  This is due to the fact that the position
of the vehicle at time $t$ depends on the positions of all outstanding
demands in $\QQ(t)$.  Thus, our approach is to lower bound the capture
fraction of the LP policy with a greedy policy:

\begin{algorithm}[H] 
  \dontprintsemicolon %
  \nocaptionofalgo %
  \KwAssumes{Vehicle is located at $(X,L)$} %
  Compute the reachability set $R(X,L)$. \; %
  Capture the demand in $R(X,L)$ with the highest $y$-coordinate using
  intercept motion. \;%
  Repeat. \;
  \caption{\bf The Greedy Path (GP) policy}
\end{algorithm}

Given a set of outstanding demands $\QQ(t)$ at time $t$, the Greedy
Path policy generates a suboptimal longest path through $\QQ(t)$.  In
addition, the vehicle position is independent of all outstanding
demands, except the demand currently being captured.  Thus, the
capture fraction of the Greedy Path policy provides a lower bound
for the capture fraction of the Longest Path policy.  We are now
able to establish the following result.

\begin{theorem}[Lower Bound for Longest Path policy]
\label{thm:LP_lower_bd}
If $L \geq vW$, then for the Longest Path policy
\[
\capfrac(\mathrm{LP}) \geq \capfrac(\mathrm{GP})
\geq \frac{1}{\sqrt{\pi \alpha}\erf(\sqrt{\alpha}) + \e^{-\alpha}},
\]
where $\alpha = \lambda W/2$ and $\erf:\real \to [-1,1]$ is the error
function.
\end{theorem}
\begin{proof}
  We begin by looking at the expression for the capture fraction.
  Notice that if $n_{\capt}(t) > 0$ for some $t > 0$, then
\begin{equation}
\label{eq:capt_prob_rearrange}
\begin{aligned}
  \limsup_{t\to+\infty}\expectation{\frac{n_{\capt}(t)}{n_{\capt}(t)+n_{\esc}(t)}}
  &= \limsup_{t\to+\infty}
  \expectation{\frac{1}{1+\frac{n_{\esc}(t)}{n_{\capt}(t)}}} \\
  &\geq \left(1+\limsup_{t\to+\infty}
    \expectation{\frac{n_{\esc}(t)}{n_{\capt}(t)}}\right)^{-1},
\end{aligned}
\end{equation}
where the last step comes from an application of Jensen's
inequality~\cite{LB:92}.  Thus, we can determine a lower bound on the
capture fraction by studying the number of demands that escape per
captured demand.

Let us study the time instant $t$ at which the service vehicle
captures its $i$th demand, and determine an upper bound on the number
of demands that escape before the service vehicle captures its
$(i+1)$th demand.  Since we seek a lower bound on the capture fraction
of the LP policy, we may consider the path generated by the Greedy
Path policy.  In addition, we consider the worst-case service vehicle
position; namely, the position $(0,L)$ (or equivalently $(W,L)$).

From the position $(0,L)$, the reachable set is 
\[
R(0,L) = \{(x,y)\in\env : vx \leq L\}.
\]
Let $R_y$ denote the reachable set intersected with
$[0,W]\times[L-y,L]$, where $y\in[0,L]$, and let $|R_y|$ denote its
area.  Then,
\[
|R_y| = \begin{cases}
\frac{y^2}{2v}, & \text{if $y \leq vW$}, \\
yW - \frac{vW^2}{2}, & \text{if $y > vW$}.
\end{cases}
\] 
An illustration of the set $R_y$ is shown in
Figure~\ref{fig:lower_bound_prf}.
\begin{figure}
\includegraphics[width=.55\linewidth]{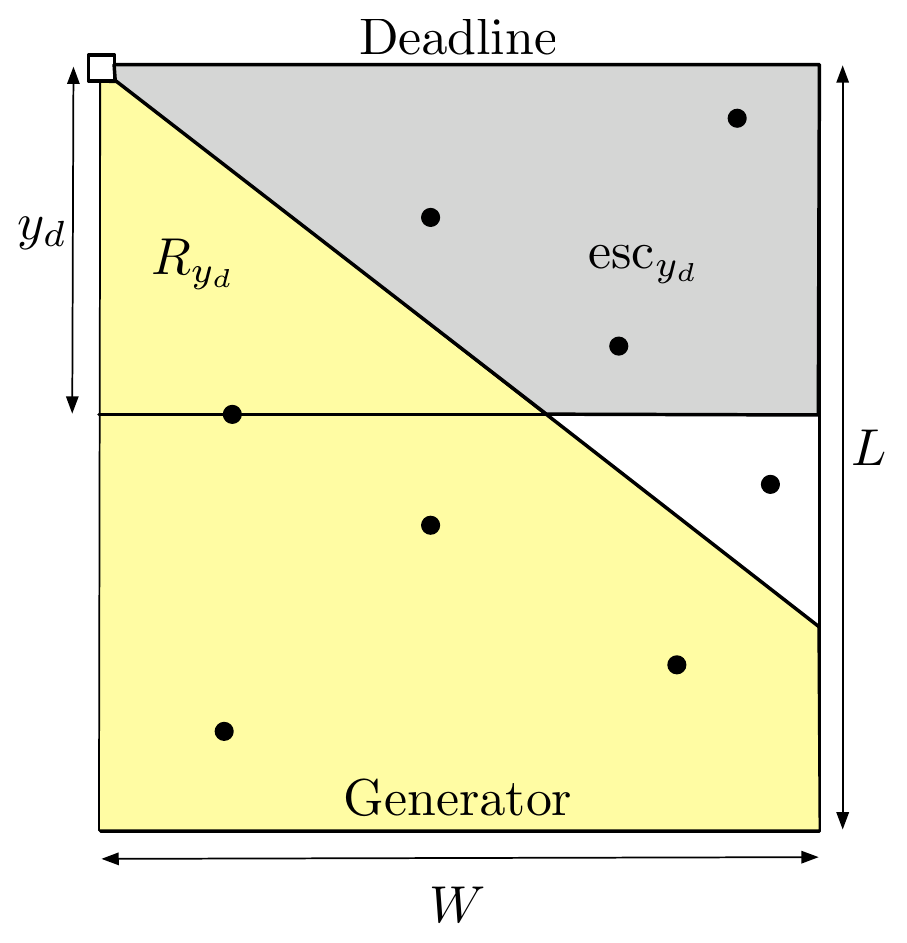}
\centering 
\caption{The setup for the proof of Theorem~\ref{thm:LP_lower_bd}.
  The service vehicle is located at $(0,L)$.  All demands in the
  region $\esc_{y_d}$ escape while capturing the demand with the
  highest $y$-coordinate.}
\label{fig:lower_bound_prf}
\end{figure}
Let $y_d$ be the $y$-distance to the reachable demand with the highest
$y$-coordinate.  That is,
\[
y_d = \min_{(x,y)\in\QQ(t) \cap R(0,L)}\{L - y\},
\]
where $\QQ(t)$ is the set of outstanding demands at time $t$.  By
Lemma~\ref{lem:poisson}, the probability that a subset $\BB \subset
\env$ with area $|\BB|$ contains zero demands is given by
\[
\prob[|\BB \cap \QQ(t)| = 0] = \e^{-\lambda |\BB|/(vW)},
\]
where $|\BB \cap \QQ(t)|$ denotes the cardinality of the finite set
$\BB \cap \QQ(t)$. Thus,
\[
\prob[y_d > y] = \prob[|R_y\cap\QQ(t)| = 0] = \e^{-\lambda |R_y|/(vW)}.
\]
The probability density function of $y_d$ for $y_d \leq vW$ is
\begin{align*}
f(y) = \frac{d}{dy}(1-\prob[y_d > y]) &= \frac{d}{dy}\e^{-\lambda y^2/(2v^2W)} \\
&= \frac{\lambda}{v^2 W}y \e^{-\lambda y^2/(2v^2W)}.
\end{align*}

Now, given $y_d$, all demands residing in the region
$\esc_{y_d}:=([0,W]\times[L-y_d,L]) \setminus R_{y_d}$ will escape (see
Fig.~\ref{fig:lower_bound_prf}).  The area of $\esc_{y_d}$ is
\[
|\esc_{y_d}| = \begin{cases}
y_dW - \frac{y_d^2}{2v}, & \text{if $y_d \leq vW$}, \\
\frac{vW^2}{2}, & \text{if $y_d \geq vW$}.
\end{cases}
\]  

From Section~\ref{sec:poisson_dist}, the expected number of
outstanding demands in an unserviced region of area $A$ is $\lambda
A/(vW)$.  Thus, given that the vehicle is located at $(0,L)$, the
expected number of demands that escape while the service vehicle is
capturing its $(i+1)$th demand is given by
\begin{align*}
&\expectation{n_{\esc,i}} = \frac{\lambda}{vW}\expectation{|\esc_{y_d}|} \\
&= \frac{\lambda}{vW}\left[\int_0^{vW}\left(yW - \frac{y^2}{2v}\right) f(y)
dy + \frac{vW^2}{2}\prob[y_d > vW]\right].
\end{align*}
Applying the probability density function and cumulative distribution
function of $y_d$ we obtain
\begin{multline}
\label{n_int_exp}
\expectation{n_{\esc,i}}=\frac{\lambda^2}{v^3W^2}
\int_0^{vW}\left(yW - \frac{y^2}{2v}\right) y \e^{-\lambda
  y^2/(2v^2W)}dy \\ + \frac{\lambda W}{2}\e^{-\lambda W/2}.
\end{multline}
To evaluate the integral, consider the change of coordinates $z :=
y/vW$, and define $\alpha:=\lambda W/2$.  After simplifying, the
integral becomes
\[
4\alpha^2\int_0^1\left(z^2 - \frac{z^3}{2}\right) \e^{-\alpha z^2}dz.
\]
Integrating by parts we obtain
\begin{equation}
\label{eq:int_sol}
\sqrt{\pi \alpha} \erf(\sqrt{\alpha}) + \alpha \e^{-\alpha} +\e^{-\alpha} -1,
\end{equation}
where $\erf:\real \to [-1,1]$ is the error function:
\[
\erf(x) = \frac{2}{\pi}\int_0^x\e^{-t^2}dt.
\]
Substituting equation~\eqref{eq:int_sol} into
equation~(\ref{n_int_exp}) we obtain
\[
\expectation{n_{\esc,i}} = \sqrt{\pi \alpha} \erf(\sqrt{\alpha})
+\e^{-\alpha} -1.
\]
Since $\expectation{n_{\esc,i}}$ is computed for the worst-case
vehicle position $(0,L)$, and since this expression holds at every
capture, we have that
\[
\limsup_{t\to +\infty}\expectation{\frac{n_{\esc}(t)}{n_{\capt}(t)}} \leq
\sqrt{\pi \alpha} \erf(\sqrt{\alpha}) +\e^{-\alpha} -1,
\]
and thus by equation~\eqref{eq:capt_prob_rearrange} we obtain the
desired result. 
\end{proof}

\section{Demand speed less than vehicle speed}\label{sec:tmhp}
\label{sec:slow_demand}

In this section we study the case when the demand speed $v < 1$.  For
this case, an upper bound on the capture fraction has been derived in
\cite{SDB-SLS-FB:08v}.  We introduce a policy which is a variant of
the TMHP-based policy in~\cite{SDB-SLS-FB:08v}, and lower bound its
capture fraction in the limit of low demand speed and high demand
arrival rate.

\subsection{Capture Fraction Upper Bound}

The following theorem upper bounds the capture fraction of every
policy for the case of $v<1$.

\begin{theorem}[Capture fraction upper bound,
  \cite{SDB-SLS-FB:08v}]
  \label{thm:lower_v_upper}
  If $v< 1$, then for every causal policy $P$
\[
\capfrac(P) \leq \min\left\{1,\frac{2}{\sqrt{v \lambda
      W}}\right\}.
\]
\end{theorem}

The proof of the above theorem is contained in~\cite{SDB-SLS-FB:08v},
and relies on a computation of the expected minimum distance between
demands. Notice that for low demands speed, i.e., $v \ll 1$, it may be
possible to achieve a capture fraction of one, even for high
arrival rates.

\subsection{The TMHP-fraction Policy}

In Section~\ref{sec:TMHP_rev} we reviewed the translational minimum
Hamiltonian Path (TMHP) through a set of demands.  The following
policy utilizes this path to service demands.

\begin{algorithm}[H] 
  \dontprintsemicolon %
  \nocaptionofalgo %
  \KwAssumes{Vehicle is located on the line $y = L/2$.} %
  Compute a translational minimum Hamiltonian path through all
  outstanding demands in $[0,W]\times[0,L/2]$, starting at the service
  vehicle position, and terminating at the demand with the lowest
  $y$-coordinate.\;%
  \eIf{\textup{time to travel entire path is less than $L/(2v)$}}%
  {%
    Service all outstanding demands by following the computed
    path. \; %
  }{%
    Service outstanding demands along the computed path for $L/(2v)$
    time units. \; %
  }%
  Repeat. %
  \caption{\bf The TMHP-fraction (TF) policy}
\end{algorithm}

Figure~\ref{fig:TMHP_fraction} shows an example of the TMHP-fraction
policy.  In contrast with the LP policy, where the vehicle remains on
the deadline, in the TMHP-fraction policy the vehicle follows the TMHP
using minimum time motion between demands as described in
Section~\ref{sec:TMHP_rev}.
\begin{figure}
\includegraphics[width=.45\linewidth]{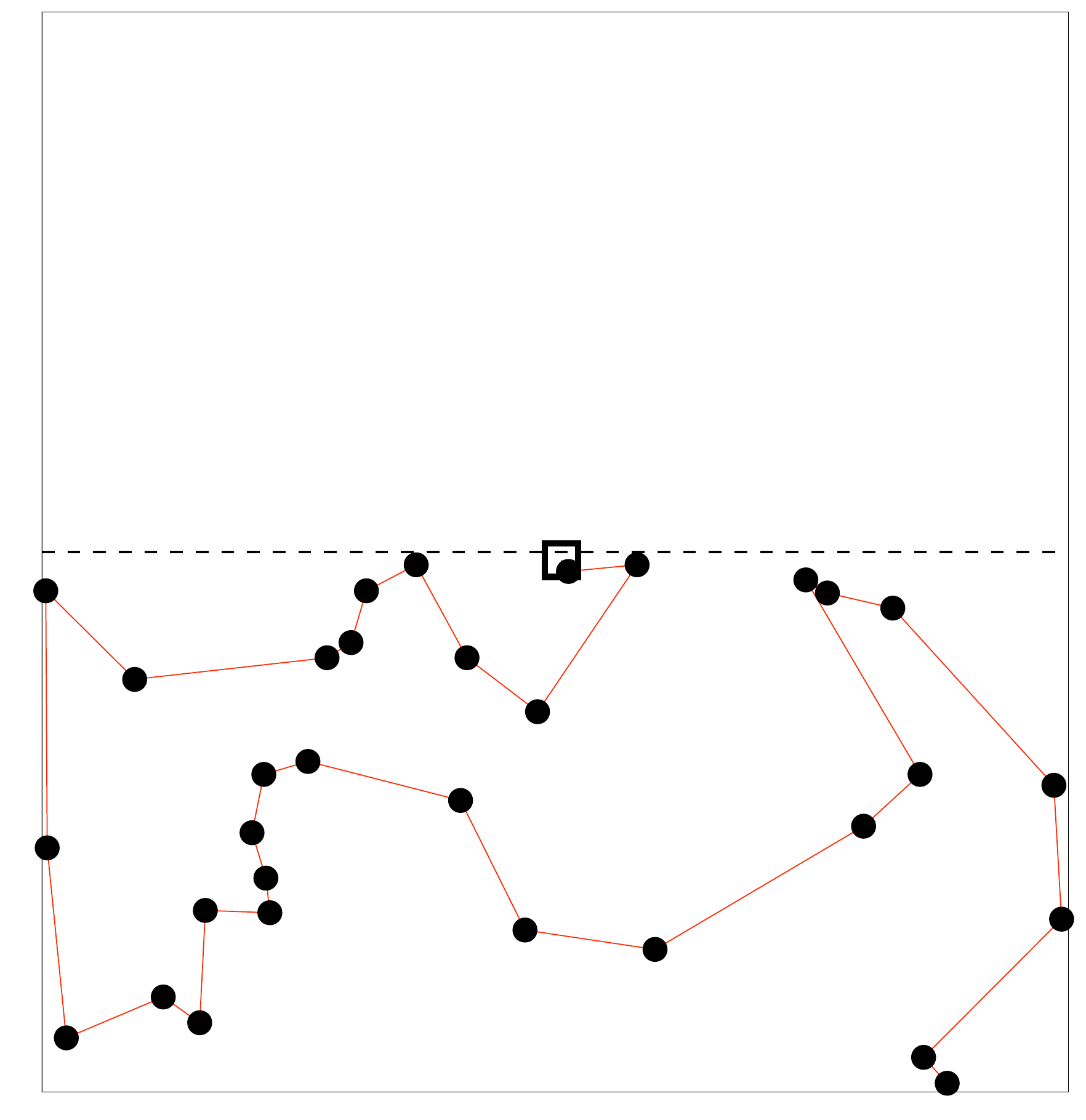}
\hfill
\includegraphics[width=.45\linewidth]{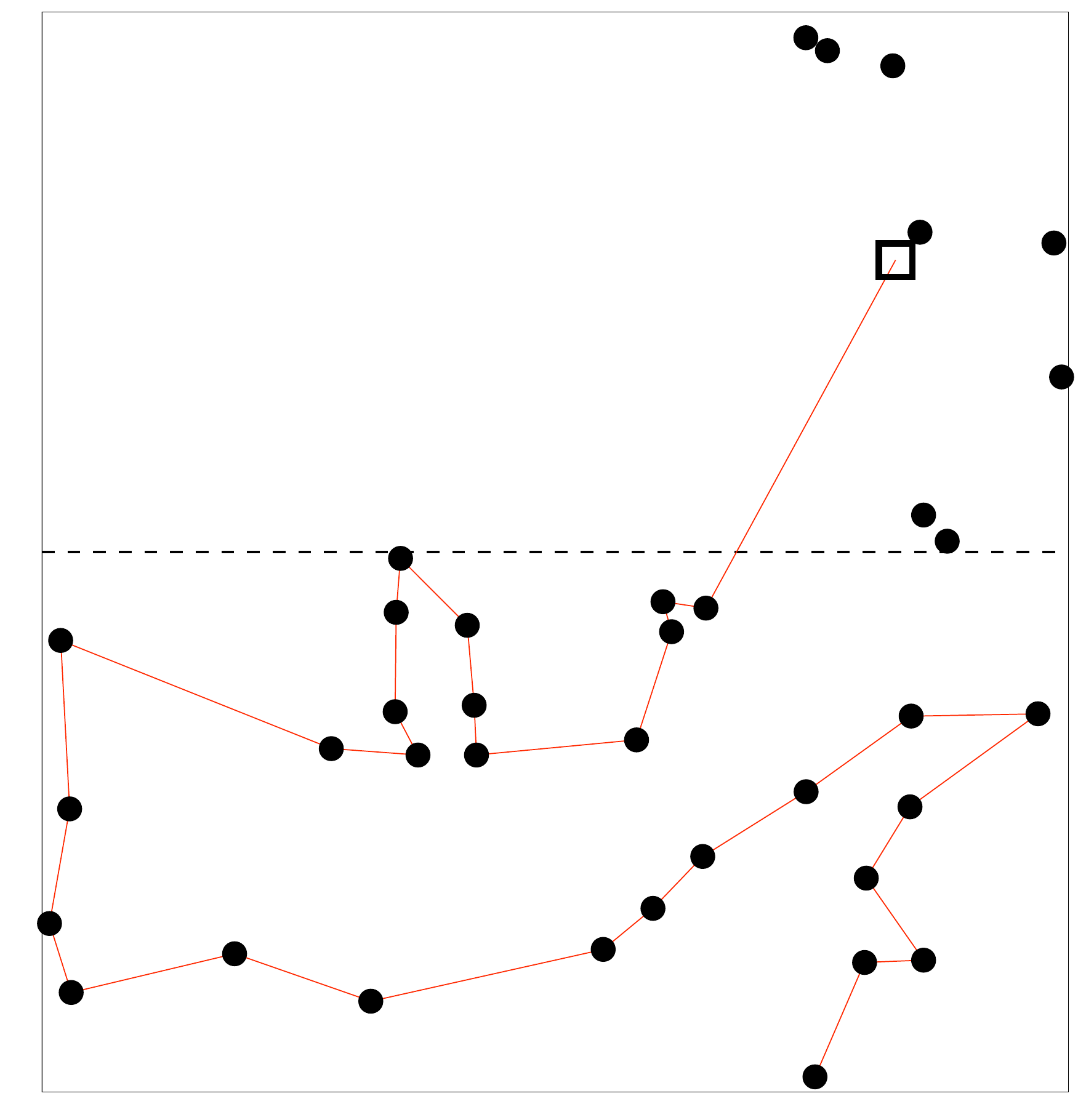}
\centering 
\caption{The TMHP-fraction policy.  The left-hand figure shows a TMHP
  through all outstanding demands.  The right-figure shows the instant
  when the vehicle has followed the path for $L/(2v)$ time units and
  recomputes its path, allowing some demands to escape.}
\label{fig:TMHP_fraction}
\end{figure}
Notice that none of the demands in the region $[0,W]\times [0,L/2]$ at
time $t$ will have escaped before time $t+ L/(2v)$.  Thus, the vehicle
is guaranteed that for the first $L/(2v)$ time units, all demands on
the TMHP path are still in the environment.  For the TMHP-fraction
policy we have the following result.

\begin{theorem}[TMHP-fraction policy lower bound]
  \label{thm:TMHP_lower_bd}
  In the limit as $v\to 0^+$ and $\lambda \to +\infty$, the capture
  fraction of the TMHP-fraction policy satisfies
\[
\capfrac(\mathrm{TF}) \geq \min\left\{1,\frac{1}{\btsp
    \sqrt{v \lambda W}}\right\}.
\]
\end{theorem}
\begin{proof}
  Consider the beginning of an iteration of the policy, and assume
  that the duration of the previous iteration was $L/(2v)$.  In this
  case, the vehicle has $y$-coordinate $Y\in[L/2,L]$, and by
  Lemma~\ref{lem:poisson}, the region $\RR:=[0,W]\times[0,L/2]$
  contains a number of demands $N$ that is Poisson distributed with
  parameter $\lambda L/(2v)$.  Conditioned on $N$, the demands are
  independently and uniformly distributed in $\RR$.

  Now, we make use of the following three facts. First, as $v\to 0^+$,
  the length of the TMHP constrained to start at the vehicle location
  and end at the lowest demand, is equal to the length of the EMHP in
  the corresponding static instance, as described in
  Lemma~\ref{lem:ttsp}. Second, from Corollary~\ref{cor:emhp_length},
  for uniformly distributed points, the asymptotic length of a
  constrained EMHP is equal to the asymptotic length of the ETSP tour.
  Third, as $v\to 0^+$, and $\lambda\to +\infty$, we have that $N$
  tends to $+\infty$ with probability one.  Using the above facts we
  obtain that the length of the TMHP starting at the vehicle position,
  passing through all demands in $\RR$, and terminating at the demand
  with the lowest $y$-coordinate, has length $\btsp\sqrt{NWL/2}$ in
  the limiting regime as $v\to 0^+$, and $\lambda\to +\infty$.

  The vehicle will follow the TMHP for at most $L/(2v)$ time units,
  and thus will service $cN$ demands, where
  \[
  c = \min\left\{1,\frac{\sqrt{L}}{\btsp v \sqrt{2 N W}}\right\}.
  \]
  
  Now, the random variable $N$ has expected value $\expectation{N} =
  \lambda L/(2v)$ and variance $\sigma^2_N = \lambda L/(2v)$.  By the
  Chebyshev inequality, 
  $\prob[|N-\expectation{N}| \geq \alpha] \leq \sigma^2_N/\alpha^2$,
  and thus letting $\alpha = \sqrt{v}\expectation{N}$, we have
  \[
  \prob[N \geq (1+\sqrt{v})\expectation{N}] \leq
  \frac{1}{v\expectation{N}} = \frac{2}{\lambda L}.
  \]
  Thus, we have
  \[
  c \geq \min\left\{1,\frac{1}{\btsp \sqrt{(1+\sqrt{v})v\lambda W}}\right\}.
  \]
  with probability at least $1-2/(\lambda L)$. In the limit as $\lambda
  \to +\infty$, with probability 1,
  \begin{equation}
    \label{eq:c_frac}
  c \geq \min\left\{1,\frac{1}{\btsp \sqrt{v\lambda W}}\right\}.
  \end{equation}
  Therefore, if the previous iteration had duration at least $L/(2v)$,
  then the total fraction of demands captured in the current iteration
  is given by equation~(\ref{eq:c_frac}).

  The other case is that the previous iteration had duration $T <
  L/(2v)$.  In this case, all outstanding demands in the region
  $\RR:=[0,W]\times[0,L/2]$ lie in a subset $[0,W]\times[0,vT]$, and
  the subset contains a number of demands $N$ that is Poisson
  distributed with parameter $\lambda T \leq \lambda L/(2v)$.  Thus,
  in this case there are fewer outstanding demands, and the bound on
  $c$ still holds.  Thus, $\capfrac(\mathrm{TF}) \geq c$, and we
  obtain the desired result.
\end{proof}

\begin{remark}[Bound comparison]
  In the limit as $v\to 0^+$, and $\lambda \to +\infty$, the capture
  fraction of the TMHP-fraction policy is within a factor of $2\btsp
  \approx 1.42$ of the optimal. \oprocend
\end{remark}

\section{Simulations}
\label{sec:simu}

We now present two sets of results from numerical experiments.  The
first set compares the Longest Path policy with $\eta =1$ to the
Non-causal Longest Path policy and to the theoretical lower bound in
Theorem~\ref{thm:LP_lower_bd}.  The second set compares the
TMHP-fraction policy to the policy independent upper bound in
Theorem~\ref{thm:lower_v_upper} and the lower bound in
Theorem~\ref{thm:TMHP_lower_bd}.

To simulate the LP and the NCLP policies, we perform $10$ runs of the
policy, where each run consists of $5000$ demands.  A comparison of
the capture fractions for the two policies is presented in
Figure~\ref{fig:LP_simu}.  When $L >vW$, the capture fraction of the
LP policy is nearly identical to that of the NCLP policy. Even in
Figure~\ref{fig:LP_simu_2}, where $L <vW$, the capture fraction of the
LP policy is within $2\%$ of the NCLP policy, and thus the optimal.
This suggests that the Longest Path policy is essentially optimal over
a large range of parameter values.
\begin{figure}
\centering
\subfigure[$v=2$ and $L > vW$.]{
\includegraphics[width=0.465\linewidth]{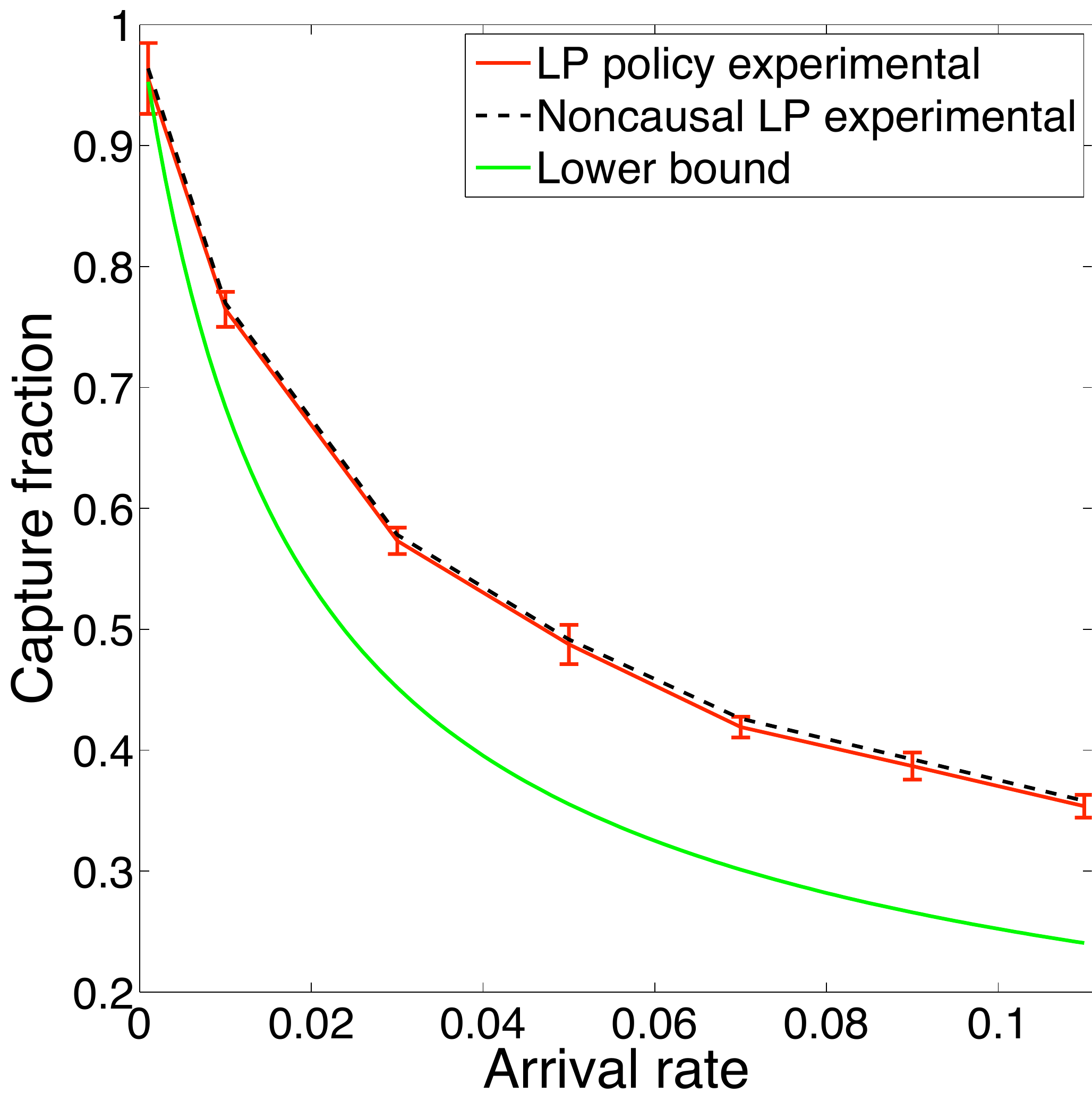}
\label{fig:LP_simu_2}
}
\subfigure[$v = 5$ and $L < vW$.]{
\includegraphics[width=0.465\linewidth]{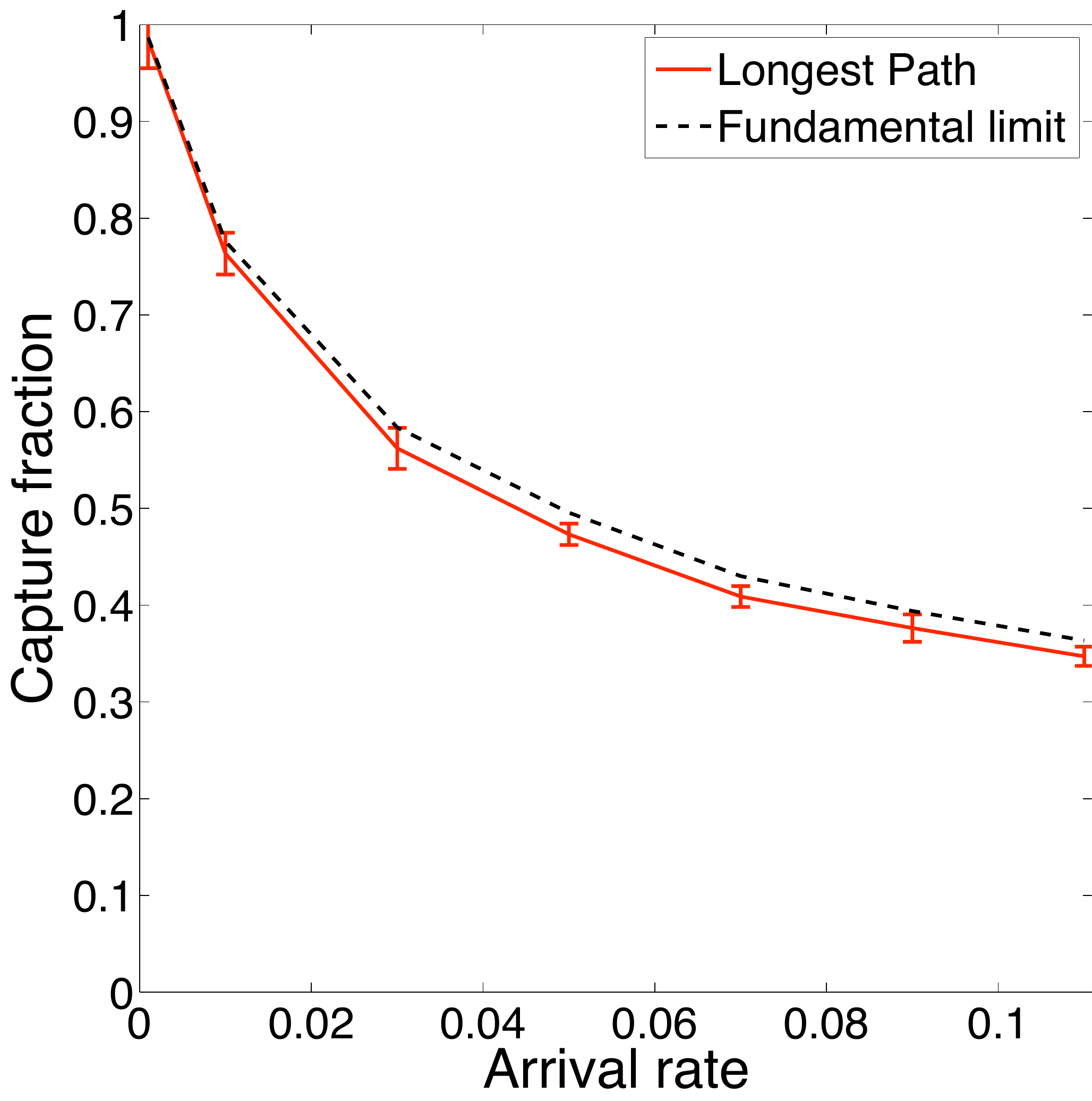}
\label{fig:LP_simu_5}
}
\caption{Simulation results for LP policy (solid red line with error
  bars showing $\pm$ one standard deviation) and the NCLP policy
  (dashed black line) for an environment of width $W=120$ and length
  $L=500$.  In (a), $L > vW$, and the lower bound in
  Theorem~\ref{thm:LP_lower_bd} is shown in solid green.}
\label{fig:LP_simu}
\end{figure}

To simulate the TMHP-fraction policy, the {\ttfamily
  linkern}\footnote{ {\ttfamily linkern} is freely available for
  academic research use at {\ttfamily
    http://www.tsp.gatech.edu/concorde.html}.} solver is used to
generate approximations to the optimal TMHP. For each value of arrival
rate, we determine the capture fraction by taking the mean over $10$
runs of the policy. A comparison of the simulation results with the
theoretical results from Section~\ref{sec:tmhp} are presented in
Figure~\ref{fig:TMHP_simu}.
\begin{figure}
\centering
\subfigure[Demand speed $v=0.01$.]{
\includegraphics[width=0.465\linewidth]{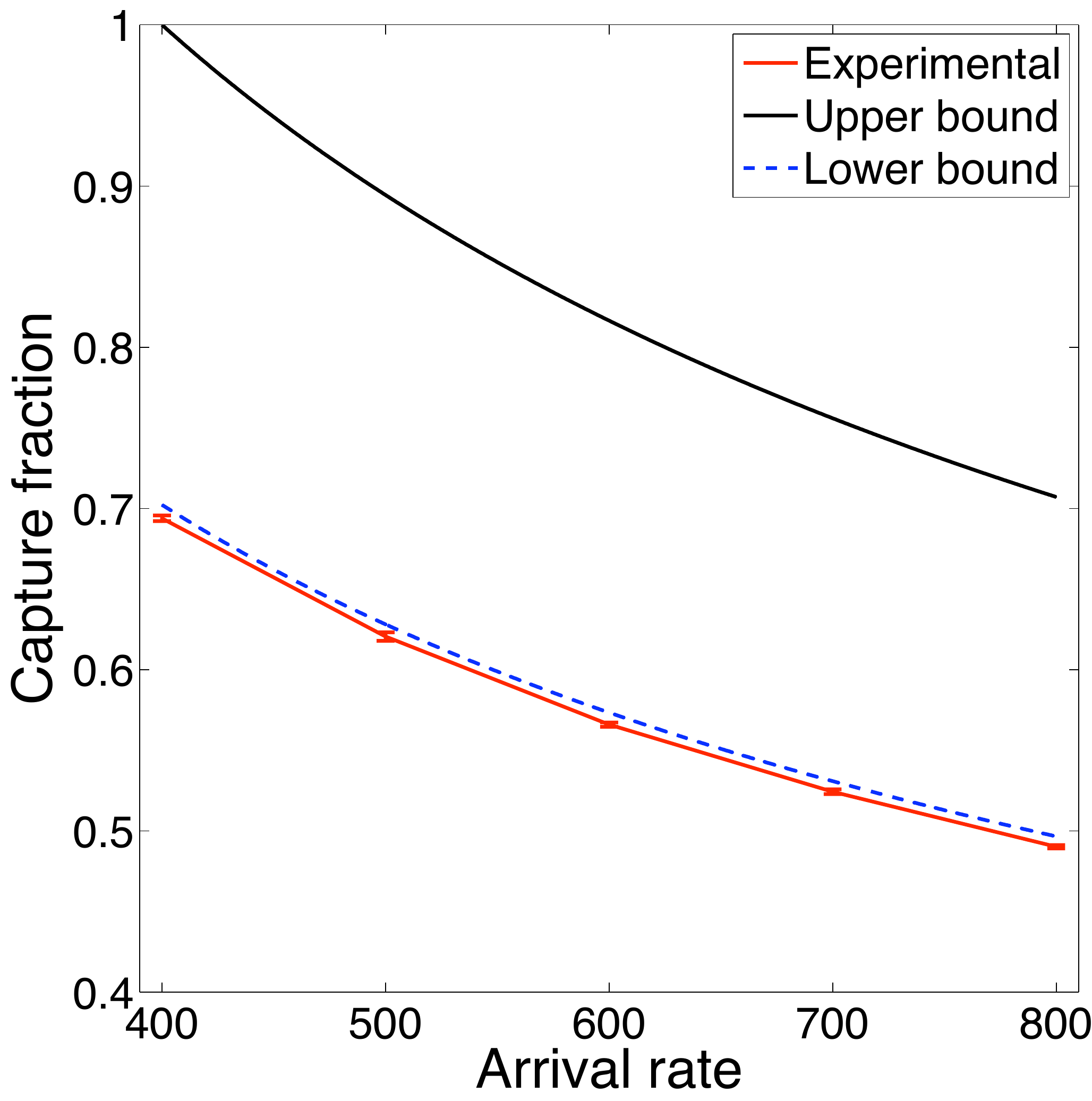}
\label{fig:TMHP_simu_01}
}
\subfigure[Demand speed $v=0.05$.]{
\includegraphics[width=0.465\linewidth]{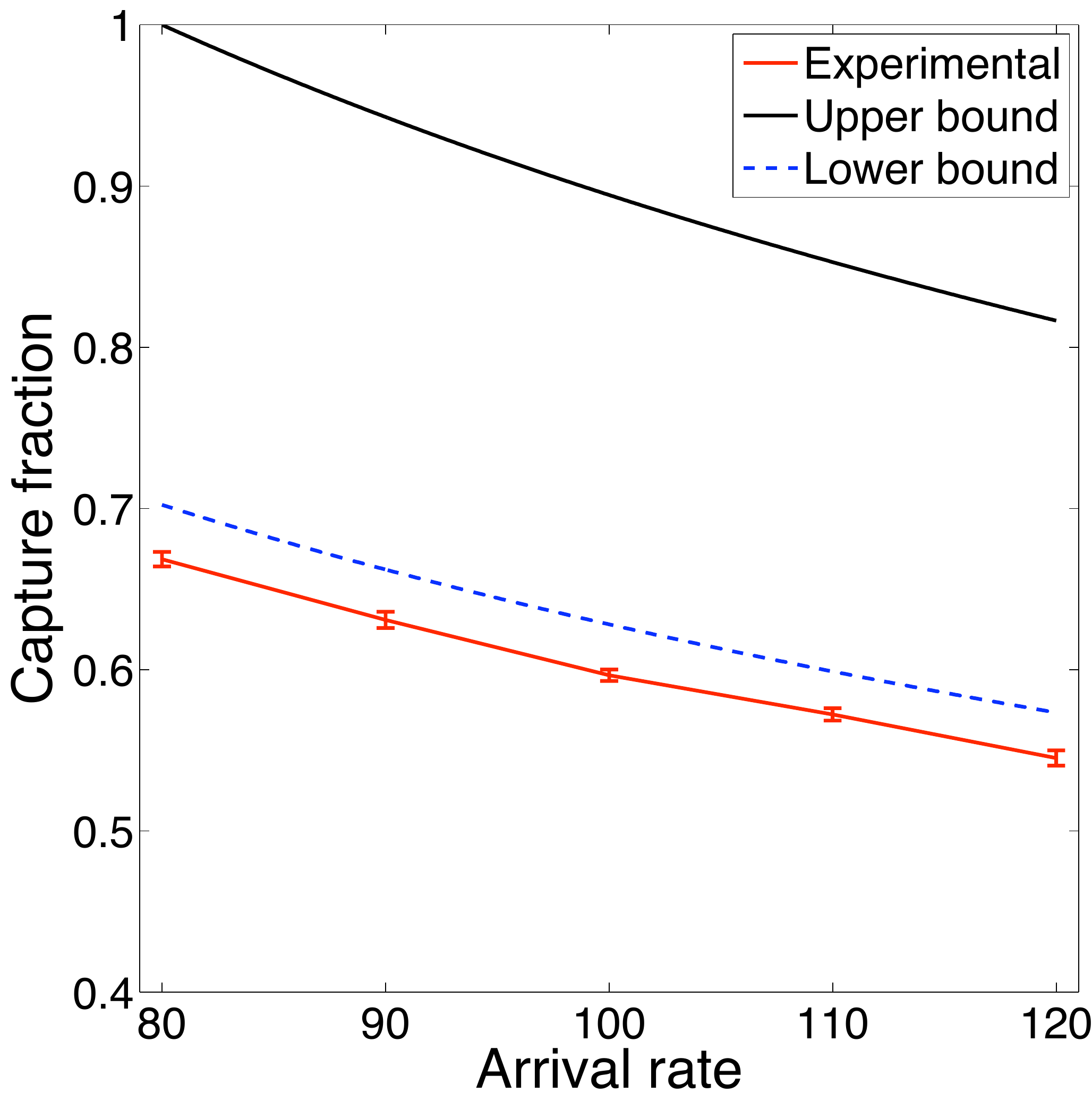}
\label{fig:TMHP_simu_05}
}
\caption{Simulation results for TMHP-fraction policy.  The solid black
  curve shows the upper bound in Theorem~\ref{thm:lower_v_upper} and
  the dashed line shows the lower bound in
  Theorem~\ref{thm:TMHP_lower_bd}.  Numerical results are shown with
  error bars.}
\label{fig:TMHP_simu}
\end{figure}
For $v=0.01$ in Fig.~\ref{fig:TMHP_simu_01}, the experimental results
are in near exact agreement with the theoretical lower bound in
Theorem~\ref{thm:lower_v_upper}.  For $v=0.05$ in
Fig.~\ref{fig:TMHP_simu_05}, the experimental results are within $5\%$
of the theoretical lower bound.  However, notice that the experimental
capture fraction is smaller than the theoretical lower bound. This is
due to several factors.  First, we have not reached the limit as $v\to
0^+$ and $\lambda \to +\infty$ where the asymptotic value of $\btsp
\approx 0.712$ holds.  Second, we are using an approximate solution to
the optimal TMHP, generated via the linkern solver.

\section{Conclusions}
\label{sec:conclusions}

In this paper we introduced a pursuit problem in which a vehicle must
guard a deadline from approaching demands. We presented novel
policies in the case when the demand speed is greater than the vehicle
speed, and in the case when the demand speed is less than the vehicle
speed. In the former case we introduced the Longest Path policy which
is based on computing longest paths in the directed acyclic
reachability graph, and in the latter case we introduced the
TMHP-fraction policy. For each policy, we analyzed the fraction of
demands that are captured.

There are many areas for future work. The Longest Path policy has
promising extensions to the case when demands have different priority
levels, and to the case of multiple vehicles. We would also like to
fully characterize the capture fraction when $L < vW$, and tighten
our existing bounds to reflect the near optimal performance shown in
simulation.

\end{document}